\documentclass{article}

\usepackage{CJK,CJKspace,CJKpunct}

\usepackage[utf8]{inputenc} 
\usepackage[T1]{fontenc}    
\usepackage{hyperref}       
\usepackage{url}            
\usepackage{booktabs}       
\usepackage{amsfonts}       
\usepackage{nicefrac}       
\usepackage{microtype}      
\usepackage{graphicx} 
\usepackage{amsmath}
\usepackage{amssymb}
\usepackage{caption}

\title{Comparing Fairness Criteria Based on Social Outcome}

\usepackage{booktabs} 
\usepackage{array} 
\usepackage{paralist} 
\usepackage{verbatim} 
\usepackage{subfig} 

\newtheorem{theorem}{Theorem}

\newtheorem{proposition}{Proposition}

\newtheorem{proof}{Proof}
\newtheorem{assumption}{Assumption}
\newtheorem{definition}{Definition}
\newtheorem{example}{Example}

\newcommand{\Prob}{\mathbb{P}}

\newcommand{\Real}{\mathbb{R}}

\newcommand{\mX}{\mathcal{X}}
\newcommand{\nn}{\nonumber\\}

\newcommand{\Normal}{\mathcal{N}}

\newcommand{\Ind}{\mathbb{I}}
\newcommand{\haty}{\hat{y}}
\newcommand{\rd}{\mathrm{d}}
\newcommand{\SW}{\mathrm{SW}}
\newcommand{\AW}{\mathrm{AW}}
\newcommand{\FW}{\mathrm{FW}}
\newcommand{\FP}{\mathrm{FP}}
\newcommand{\TP}{\mathrm{TP}}

\usepackage{color}

\author{
  Junpei Komiyama \\
  The University of Tokyo\\
  \texttt{junpei@komiyama.info} \\
  \and
  Hajime Shimao \\
  Santa Fe Institute \\
  \texttt{hajime.fr@gmail.com} \\
}
\date{}

\pdfmapline{=unisong@Unicode@ <ipam.ttf}

\begin{document}
\begin{CJK*}{UTF8}{zhsong}

\maketitle

\begin{abstract}
Fairness in algorithmic decision-making processes is attracting increasing concern. When an algorithm is applied to human-related decision-making an estimator solely optimizing its predictive power can learn biases on the existing data, which motivates us the notion of fairness in machine learning.
while several different notions are studied in the literature, little studies are done on how these notions affect the individuals. 
We demonstrate such a comparison between several policies induced by well-known fairness criteria, including the color-blind (CB), the demographic parity (DP), and the equalized odds (EO). We show that the EO is the only criterion among them that removes group-level disparity. Empirical studies on the social welfare and disparity of these policies are conducted.
\end{abstract}

\section{Introduction}
\label{sec_intro}
The goal of the supervised learning is to estimate label $y$ by learning an estimator $\haty(X)$ as a function of associated feature $X$. 
Arguably, an estimator of better predictive power is preferred, and standard supervised learning algorithm learns $\haty(X)$ from existing data.
However, when it is applied to human-related decision-making, such as employment, college admission, and credit, an estimator optimizing its predictive power can learn biases on the existing data. 
To address this issue, fairness-aware machine learning proposes methodologies that yield predictors that not only have better predictive power but also complies with some notion of non-discrimination. 

Let $s$ be the (categorical) sensitive attribute among $X$ that represent the applicants' identity (e.g., gender or race). 
Group level fairness concerns the inequality among groups of different $s$.
A naive approach, which we call \textit{color-blind} \cite{CL93}, is to remove $s$ from $X$ in predicting $\haty$: Although such an approach avoids direct discrimination through $s$, the correlation between $s$ and the other attributes in $X$ causes indirect discrimination, which is referred to as the disparate impact. 
Another notion of fairness, which is widely studied (e.g., \cite{DBLP:conf/pkdd/KamishimaAAS12,DBLP:conf/cikm/RistanoskiLB13,DBLP:conf/icml/ZemelWSPD13,DBLP:conf/pkdd/FukuchiSK13}), is demographic parity (DP). DP requires the independence of $s$ from $\haty$. For instance, a university admission comply with DP if each group has equal access to the university. The demographic parity is justified in the legal context in labor market: The U.S. Equal Employment Opportunity Commission \cite{eeoc} clarified the so-called 80\%-rule, that prohibits employment decisions of non-negligible inequality. 
In spite of such legal background, some concerns on DP are raised. Hardt et al.\,\cite{DBLP:conf/nips/HardtPNS16} argued that DP is incompatible with the perfect classifier $\haty = y$, and thus it is not appropriate when the true label $y$ is reliable. To address this issue, they provided an alternative notion of fairness called the equalized odds (EO), which requires the independence of $s$ from $\haty$ conditioned on $y$ and thus allows $\haty = y$. Note that essentially the same notion is also proposed in Zafar et al.\,\cite{DBLP:conf/www/ZafarVGG17}, and the notion of the counterfactual fairness \cite{KusnerLRS17} is similar to EO given a specific causal modeling. Note that DP and EO are mutually incompatible \cite{KleinbergMR17}.

Despite massive interest in the fairness in machine learning, only a few of them concerned on the resulting social impact of a policy based on the proposed notion of fairness produces. The result of a policy is far from straightforward: In some case, an introduction of a naive notion of fairness can be harmful:
For example, consider the case of a university admission policy. If the admission office discriminates blacks by believing they are less likely to perform well academically and lowers their admission standard for them to propel affirmative action, blacks may be discouraged to invest in their education because they pass the admission regardless of their effort. As a result, blacks may end up being less proficient and the negative stereotype ``self-perpetuates''.
Indeed, self-fulfillment of stereotype is an empirically documented phenomenon in some fields \cite{glover2017discrimination}.
The difficulty of analyzing this phenomenon lies in the interaction between the policy-maker and the applicants: when a policy changes, the applicants also change their behavior due to a modified incentive.

This lack of interest in the social outcome, in turn, results in the absence of a unified measure to compare different fairness criteria. 
In this regard, economic theory offers useful tools. In particular, literature in labor economics has a long history of analysis of welfare implication of policy changes. That is, economists investigate how the players' welfare, or aggregate level of their utility, changes by imposing a policy. 

By combining the theoretical framework developed in labor economics with the ``oblivious'' post-processing non-discriminatory machine learning \cite{DBLP:conf/nips/HardtPNS16}, we propose a framework of comparison between different fairness notions in view of the incentives.
We demonstrate such a comparison between several policies induced by well-known fairness criteria; color-blind (CB), demographic parity (DP), and equalized odds (EO). As a result, we show that while CB and DP sometimes disproportionately discourage unfavored groups from investing in the improvement of their value, EO equally incentivizes the two groups. 

Importantly, our framework is not just theoretical but applicable to practices and enables to assess the fairness notions based on the actual situation.
To demonstrate this point, we compare the fairness policies by using a real-world dataset: We show that (i) Unlike CB and DP, EO is disparity-free. Moreover, (ii) all of the CB, DP, and the EO tend to reduce social welfare compared to no fairness intervention case. Among them, EO yielded the lowest social welfare: One can view this as a cost of removing disparity. 

\subsection{Related work}


A long line of works on discrimination and affirmative action policy exists in the literature of labor economics(\cite{arrow1998has};\cite{holzer2000assessing};see Fang and Moro\,\cite{fm2011} for a survey of recent theoretical frameworks). Coate and Loury \cite{CL93} considered a simple model where an employer infers applicants' productivity based on one-dimensional signal, which contains information about their invested effort in skill. This nominal paper argues that even under the affirmative action policy to enforce the employer to set the same rate of hiring to all the groups, there still exist equilibria where one group is negatively stereotyped, and consequently, discouraged from investing in skills.

The problem of those analyses in economics is that their setting is abstract and simplified so that they do not allow us real-world applications with actual datasets. For instance, based on their simple model, Coate and Loury \cite{CL93} states that ``The simplest intervention would insist that employers make color-blind assignments'' and it would ensure the fairness as well as the same incentives across groups. However, it is commonly perceived in machine learning that color-blind policy does not ensure fairness due to disparate impact \cite{sweeney2013,misc:219,pmlr-v81-buolamwini18a}. 
Due to a lack of consideration on such learning-from-data process and related issues, frameworks proposed in economics are not designed for the real-world application. This paper modifies their models to be applicable to machine learning problems.
More importantly, their main interest lies in affirmative action: 
While affirmative action that imposes a restriction on the outcome such as the ratio of admitted students (which is similar to demographic parity) is arguably important, modern machine learning algorithms propose various methodologies to ensure the fairness at the prediction level, not the outcome level.

A few papers in machine learning considered a game-theoretic view of decision-making processes and thus enable us to compare fairness criteria. In particular, the closest papers to ours are \cite{HuC18,delayedarxiv}. Hu and Chen\,\cite{HuC18} considered a two-stage process and each stage dealt with group-level and individual-level fairness, whereas we are focusing on comparing the several notions of group-level fairness.
Liu et al.\,\cite{delayedarxiv} compared several notions of fairness including the demographic parity and the equalized opportunity in terms of its long-term improvements and characterized the conditions where each of these fairness-related constraints works. Unlike ours, the analysis in Liu et al.\,\cite{delayedarxiv} assumes the availability of the function that determines how the delayed impact from the prediction arises. Identification of such a function requires us counter-factual experiments or model-dependent analyses.
Moreover, they evaluate the fairness criteria by the disparity between groups, without analyzing the social welfare.
By assuming a model with micro-foundation of players' decision-making, we are able to compare the welfare implication of different fairness criteria.

\section{Model}

\begin{figure}
\centering
\includegraphics[width=0.8\textwidth]{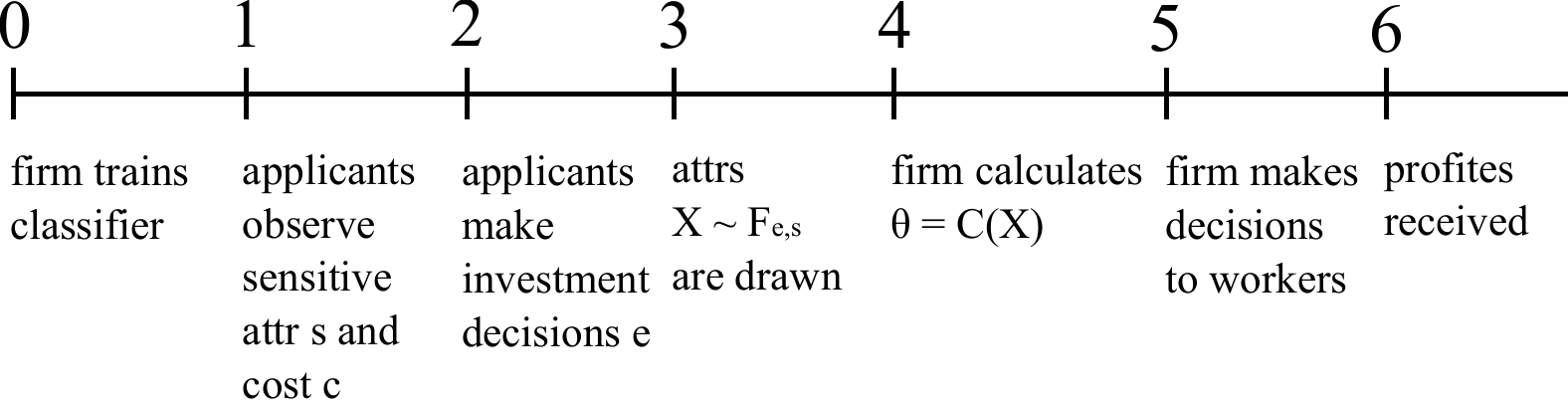}
\caption{Sequence of timings.}
\label{fig_timeline}
\end{figure}%

We consider a game between a continuum of applicants and a single firm. 
The game models application processes, such as university admissions, job applications, and credit card applications.
A firm has a dataset on the performance of applicants and uses it to estimate the performance of the future applicants. For the ease of discussion, we assume that there exist two groups:
Applicant of each group is assigned a sensitive attribute $s \in \{0,1\}$. Let $\lambda_1$ be the fraction of the applicants of $s=1$, and $\lambda_0$ be $1- \lambda_1$. Each of the applicants has an option to exert his or her effort, and before determining whether or not to exert the effort, the applicant is given a cost $c \in [\underline{c},\bar{c}]$ of that. Let $e \in \{q,u\}$ be the variable that indicates the effort of an applicant. The applicant's feature $X \in \mX$ is drawn from a distribution.
The effort is very relevant to the performance of the applicant, and thus the firm would like to admit all the applicants of $e=q$ (that we call the qualified applicants) and to dismiss the applicants of $e=u$ (that we call the unqualified applicants). If a qualified applicant is accepted, the firm earns revenue $v_q>0$. If an unqualified applicant is accepted, the firm loses $-v_u<0$ (= negative revenue). All the applicants prefer to be accepted, and let $\omega$ be the revenue of the applicant to be accepted. The firm uses the pre-trained classifier that estimates the effort $e$ of the applicant from the sensitive attribute $s$ and non-sensitive attributes $X$. 
Following \cite{DBLP:conf/nips/HardtPNS16}, we assume that the classifier is a function $\mX \rightarrow \Real$, where $\theta(X) \in \Real$ indicates how likely the applicant is qualified.
Let $f_{e,s}(\theta)$ and $F_{e,s}(\theta)$ be the density and distribution of $\theta = \theta(X)$ given $e$ and $s$. 
Let $G_s(c)$ be the distribution of the cost $c$ given $s$. 
For the ease of discussion, we assume $G_s(c)$ be a uniform distribution over $[\underline{c}, \bar{c}]$.
Figure \ref{fig_timeline} displays the timing of the interaction between the applicants and the firm.
We pose the following assumption on the signal $\theta$ of the classifier.
\begin{assumption}{\rm Monotone Likelihood Ratio Property (MLRP): }
$\frac{f_{q,s}(\theta) }{f{u,s}(\theta)}$ is strictly increasing in $\theta$ for $s=\{0,1\}$.
\label{asm_mlrp}
\end{assumption}%
Namely, Assumption \ref{asm_mlrp} states that the applicant of a larger $\theta$ is more likely to be qualified. 

In the sequel, we discuss rational behavior of the firm (Section \ref{subsec_firm}) and the applicants (Section \ref{subsec_applicant}). 

\subsection{Firm's behavior}
\label{subsec_firm}

The MLRP (Assumption \ref{asm_mlrp}) motivates the firm to make a threshold of $\theta$ on the hiring decision. 
 A rational firm, without fairness-related restriction, would optimize its revenue, and the optimal threshold of $\theta$ depends on the firm's belief on the fraction of the qualified applicants:
Let $\pi_s$ be the fraction of the qualified applicants given $s$. 

When the firm observes $(\theta,s)$, the probability of this applicant being qualified is
\begin{align*}
\Prob(e=q|\theta,s) = \frac{\pi_s f_{q,s}(\theta)}{\pi_s f_{q,s}(\theta)+(1-\pi_s) f_{u,s}(\theta)}
\end{align*}
The firm accepts this applicant iff $\Prob(e=q|\theta,s)v_q+(1-\Prob(e=q|\theta,s))v_u\geq 0 $.
Given the MLRP assumption, this is equivalent to set a threshold $\tilde{\theta}_s$ such that
\begin{equation}
\frac{v_q}{v_u} = \frac{1-\pi_s}{\pi_s} \frac{f_{u,s}(\tilde{\theta}_s)}{f_{q,s}(\tilde{\theta}_s)} 
\label{eq_rfirm}
\end{equation}
Letting $r = v_q/v_u$ and $\phi_s = f_{u,s}(\tilde{\theta}_s)/f_{q,s}(\tilde{\theta}_s)$, \eqref{eq_rfirm} is equivalent to 
\begin{equation}
\pi_s = \frac{\phi_s(\tilde{\theta}_s)}{r + \phi_s(\tilde{\theta}_s)},
\label{eq_rfirmtwo}
\end{equation}
and the applicants with $\theta >\tilde{\theta}_s$ is approved.

\subsection{Applicants' behavior}

\label{subsec_applicant}
 Let $\tilde{c}_s(\theta) = \omega [ F_{u,s}(\theta) - F_{q,s}(\theta) ]$
be the expected increase of reward by exerting an effort. Given the firm's threshold $\tilde{\theta}_s$,  $\tilde{c}_s(\tilde{\theta}_s)$ is the incentive of the applicant to exert an effort.
A rational applicant invests in skills iff his or her cost $c$ is smaller than $\tilde{c}_s(\tilde{\theta}_s)$ : which implies
\begin{equation}
\label{eq_applicants}
 \pi_s = G(\tilde{c}_s(\tilde{\theta}_s)) := \min\left(1, \frac{\tilde{c}_s(\tilde{\theta}_s) - \underline{c}}{\bar{c}-\underline{c}}\right).
\end{equation}

\subsection{Laissez-faire Equilibria}
\label{subsec_lf}

\begin{figure}
\centering
\includegraphics[width=0.4\textwidth]{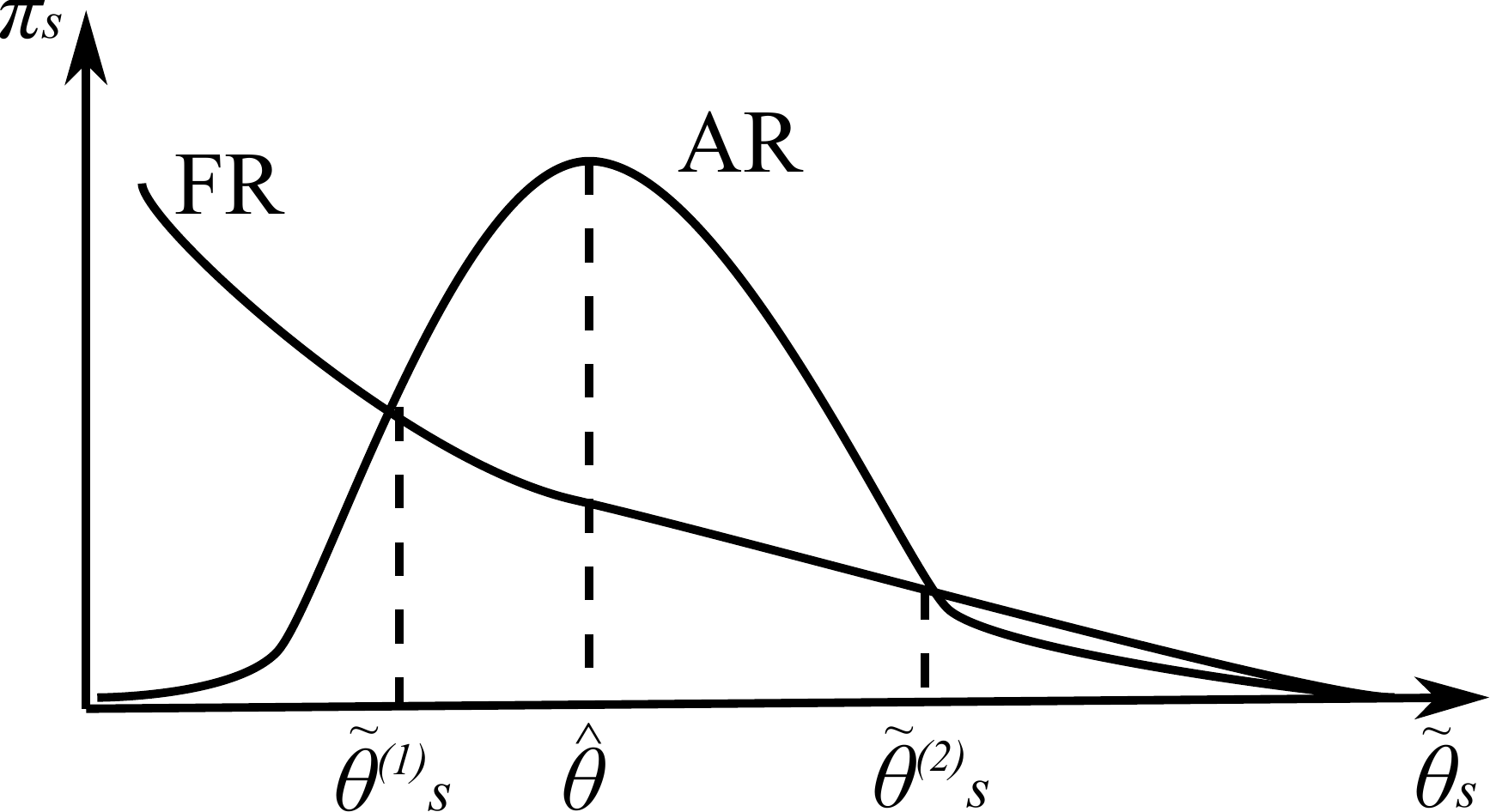}
\caption{Illustration of the equilibrium parameters $(\tilde{\theta}_s, \pi_s)$. Assumption \ref{asm_mlrp} implies that the FR curve is strictly decreasing, and the AR curve is unimodal. The value $\hat{\theta} = 
\left[ \theta : (\rd G(\tilde{c}(\theta)))/(\rd \theta) = 0\right]$ is the mode of the AR curve.}
\label{fig_eewwcurve}
\end{figure}%

Section \ref{subsec_firm} (resp. \ref{subsec_applicant}) introduced the best response of the firm (resp. the applicants) to the belief in response to the action of the applicants (resp. the firm).
When no fairness-related constraint is posed, a firm that fully exploits $s$ (that we call ``Laissez-faire'', LF) will set different threshold $\tilde{\theta}_s$ for each $s$.
If the fraction of the qualified people and the threshold of hiring are exactly the two rate postulated by the beliefs, then the players on both sides cannot increase their revenue by deviating from the current actions: Namely, in the equilibrium $\pi_s = G(\tilde{c}_s(\tilde{\theta}_s))$ holds:
\begin{definition}{\rm (Laissez-Faire Equilibrium \cite{CL93})} 
An equilibrium is a quadraple $(\tilde{\theta}_0, \tilde{\theta}_1, \pi_0, \pi_1)$ satisfying Equality \ref{eq_rfirmtwo} for $s=0,1$ and $\pi_0 = G(\tilde{c}_0(\tilde{\theta}_0)), \pi_1 = G(\tilde{c}_1(\tilde{\theta}_1))$.
\end{definition}
Figure \ref{fig_eewwcurve} illustrates the beliefs $\pi_s$ on equilibria, which is the intersections of the following two curves. Namely, (i) the Firm-Response (FR) curve: $\{(\tilde{\theta}_s,\pi_s): \pi_s = \frac{\phi_s(\tilde{\theta}_s)}{r + \phi_s(\tilde{\theta}_s)} \}$ that indicates the threshold that maximizes the firm's revenue and (ii) the Applicant-Response (AR) curve: $\{(\tilde{\theta}_s,\pi_s): \pi_s = G(\tilde{c}_s(\tilde{\theta}_s))\}$ that indicates the incentive of the applicants. 
The following proposition holds:
\begin{proposition}{\rm (Existence of multiple equilibria, Proposition 1 in Coate and Loury \cite{CL93})}
For each $s$, there exist two or more intersections of the FR and AR curves if and only if there exists $\tilde{\theta}_s$ such that $ G(\tilde{c}_s(\tilde{\theta}_s)) > \frac{\phi_s(\tilde{\theta}_s)}{r + \phi_s(\tilde{\theta}_s)}$.
\end{proposition}
The proof directly follows from the monotonicity of FR and the unimodality of AR.
As discussed by Coate and Loury \cite{CL93}, the existence of multiple intersections implies the existence of asymmetric equilibria where $\pi_0 < \pi_1$, even in the case signal is not biased (i.e., $F_{e, s=0}(\theta) = F_{e,s=1}(\theta)$). Such an asymmetric equilibrium discourages the unfavored group $s=0$ as $\theta_0 < \theta_1$ implies the reduced incentive of the unfavored group. 

\subsection{Social Welfare}
\label{subsec_sw}

In accordance with Sections \ref{subsec_firm} and \ref{subsec_applicant}, we define the social welfare as follows:
The firm's welfare is
\begin{align*}
\FW_s = \FW_s(\theta, \pi) = \left( \pi_s (1-F_{q,s}(\theta)) v_q - (1-\pi_s) (1-F_{u,s}(\theta)) v_u \right),
\end{align*}
whereas the applicants' welfare is
\begin{align*}
\AW_s = \AW_s (\theta, \pi) = \omega \Bigl( \pi \left( 1-F_{q,s}(\theta) \right) + (1 - \pi) \left( 1-F_{u,s}(\theta) \right) \Bigr) + \int_{\underline{c}}^{(1 - \pi) \underline{c} + \pi \bar{c}} c \rd c. 
\end{align*}
The social welfare is the sum of the two quantities above summed over the groups: Let $\SW_s = \FW_s + \AW_s$. The quantity $\SW = \sum_{s} \lambda_s \SW_s(\tilde{\theta}_s, \pi_s)$ is the social welfare per applicant.

\begin{theorem}{\rm (Equilibrium of the maximum social welfare)}
Fix $s \in \{0,1\}$. For group $s$, let there be two equilibria $(\tilde{\theta}_s^{(1)}, \pi_s^{(1)})$, $(\tilde{\theta}_s^{(2)},\pi_s^{(2)})$ such that $\pi_s^{(1)} > \pi_s^{(2)}$. Let $\SW_s^{(1)}, \SW_s^{(2)}$ the corresponding social welfares. Then, $\SW_s^{(1)} > \SW_s^{(2)}$.
\label{thm_sworder}
\end{theorem}
\begin{proof}
Note that, the fact that $(\tilde{\theta}_s^{(i)},\pi_s^{(i)})$ for $i \in \{1,2\}$ are at equilibrium implies that
\begin{equation} 
\SW_s^{(i)} := \FW_s(\tilde{\theta}_s^{(i)},\pi_s^{(i)}) + \AW_s(\tilde{\theta}_s^{(i)},\pi_s^{(i)}) = \max_{\theta} \FW_s(\theta, \pi_s^{(i)}) + \max_\pi \AW_s(\tilde{\theta}_s^{(i)}, \pi).
\label{ineq_sw_sup}
\end{equation}
, and thus
\begin{multline}
\SW_s^{(1)} - \SW_s^{(2)} \ge \min_\theta \left( (\FW_s(\theta, \pi_s^{(1)}) - \FW_s(\theta, \pi_s^{(2)})) \right) \\ 
+ \max_\pi \AW_s(\tilde{\theta}_s^{(1)}, \pi) - \max_\pi \AW_s(\tilde{\theta}_s^{(2)}, \pi).
\end{multline}
The term $\min_\theta \left( (\FW_s(\theta, \pi_s^{(1)}) - \FW_s(\theta, \pi_s^{(2)})) \right)$ is positive because $\FW_s(\theta, \pi)$ is strictly increasing in $\pi$. 
On the other hand, the monotonicity of the FR curve and $\pi_s^{(1)} > \pi_s^{(2)}$ imply $\tilde{\theta}_s^{(1)} < \tilde{\theta}_s^{(2)}$. The second term $\max_\pi \AW_s(\tilde{\theta}_s^{(1)}, \pi) - \max_\pi \AW_s(\tilde{\theta}_s^{(2)}, \pi)$ is non-negative because $\max_\pi \AW_s(\theta, \pi)$, which is the function of $\theta$, is decreasing it is a integration over applicants and each applicant takes maximum over (i) pay its cost $c$ to get reward $\omega (1 - F_{q,s}(\theta))$ or (ii) get reward $\omega (1 - F_{u,s}(\theta))$, and both of the two options have decreasing reward in $\theta$.
\end{proof}
Theorem \ref{thm_sworder} states that the equilibria are ordered by $\pi$: This matches our conception on the application process. The more effort the applicants pay, the more applicants the firm accepts, and the better the equilibrium is.

\section{Fairness Criteria and Their Results}
\label{sec_fairpolicy}

Section \ref{subsec_lf} shows that a lack of fairness constraint discourages the individuals of the unfavored group under an asymmetric equilibrium. A natural question is that, whether or not we can impose some non-distriminatory constraint on the firm's decision-making to remove such asymmetric equilibria.
This section compares several constraints that are discussed in the literature.

The first constraint is the one that adopts the same threshold to the two groups:
\begin{definition}{\rm (Color blind (CB) policy)}
The firm decision is said to be color-blind iff 
$\tilde{\theta}_0 = \tilde{\theta}_1$. 
The equilibria under CB are characterized by a set of quadraples $(\tilde{\theta}_0, \tilde{\theta}_1, \pi_0, \pi_1)$ that satisfies following constraints: (i) Equality \eqref{eq_applicants} holds for $s=0,1$. (ii) Moreover, letting $\tilde{\theta} := \tilde{\theta}_0 = \tilde{\theta}_1$, the following holds: 
\[
(\lambda_0 \pi_0 + \lambda_1 \pi_1) = \frac{\phi(\tilde{\theta})}{r + \phi(\tilde{\theta})},
\]
where $\phi(\theta) = \frac{\lambda_0 f_{u,s=0}(\theta) + \lambda_1 f_{u,s=1}(\theta)}{\lambda_0 f_{q,s=0}(\theta) + \lambda_1 f_{q,s=1}(\theta)} $.
\end{definition}
In other words, under CB the firm considers an optimization of single $\tilde{\theta}$ over a single group that mixed the two groups of $s=0,1$. 
Contrary to the argument of \cite{CL93} (as discussed in Section \ref{sec_intro}),
CB potentially yields an unfair treatment between two groups when $F_{e,s}(\theta)$ varies largely among two groups $s=0,1$:
\begin{proposition}
There exists an equilibrium with $\pi_0 \ne \pi_1$ under CB.
\label{prop_cbdisp}
\end{proposition}
In the following, we show examples of the disparity in Proposition \ref{prop_cbdisp}.
Let $\Normal(\mu, \sigma^2)$ be a normal distribution with mean $\mu$ and variance $\sigma^2$.
Let $\Ind(A)$ be $1$ if $A$ holds or $0$ otherwise, 

\begin{example}{\rm (Insufficient identification)}
Let $d=1$ and 
\begin{align}
X_{s=0} &\sim \Normal(\Ind(e=q), 1)\nn
X_{s=1} &\sim \Normal(\Ind(e=q)+10, 1)
\end{align}
and $\lambda_0 \approx 1$. As the classifier cannot consider $s$ explicitly, it utilizes the only dimension as $\theta = X$. 
Assume that $v_u, v_q$, and $\omega$ are such that there exists more than two equilibria as shown in Figure \ref{fig_eewwcurve} for group $s=0$. Remember that the equilibria under CB is determined by the interaction between the firm and a mixture of two groups $s=0,1$.
As the population of $s=1$ approaches $0$, one can show that the $\tilde{\theta}$ of any equilibrium is arbitrarily close to the one of the equilibria for the majority $s=0$, which has some capability of identifying $e=u$ or $q$ of person in $s=0$, and thus $\tilde{\theta}$ is not very far from $0.5$.
In this case, most people of $s=1$ would be assigned to $\hat{y}=1$ regardless of their efforts (which discourages them), and thus $\pi_1$ is close to $0$ whereas $\pi_0$ is not.
\end{example}
Another example is the case where predictive power of $\theta$ largely differs between two 
\begin{example}{\rm (signal of different accuracy)}
Let $X \in \Real^2$ and $\textbf{b}_0, \textbf{b}_1$ be the orthogonal bases of $X$.
\begin{align}
X|s=0 &\sim \Normal(\Ind(e=q)-0.5, 1)\ \textbf{b}_0  \nn
X|s=1 &\sim \Normal(\Ind(e=q)-0.5, 10^2)\ \textbf{b}_1.
\end{align}
In this case, a linear classifier can utilize a linear combination of the two basis to create a signal $\theta$: The first (resp. the second) basis is for identifying the effort of people in $s=0$ (resp. $s=1$). For any threshold value of $\theta$, such a signal is of very different incentive $\tilde{c}_s(\theta)$ between groups $s=0,1$. Due to the noisy signal, $\theta$ gives very little information on whether a person of $s=1$ exert an effort or not. When an equilibrium exists, the very little of $s=1$ would exert an effort, whereas a certain portion of $s=0$ would be incentivized to exert an effort.
\end{example}%
The implication of the examples above is as follows: When the signal $\theta$ treats the two group differently, as is shown in the case of credit risk prediction \cite{DBLP:conf/nips/HardtPNS16} (Figure 4 therein), the accuracy of a classifier can vary among $s$, which will make a mere application of CB fail.

We next consider the constraint of the demographic parity, which is arguably the most common notion of fairness in the context of fairness-aware machine learning.
\begin{definition}{\rm (demographic parity, DP)} 
The firm decision is said satisfy demographic parity iff
$\Prob[\theta > \tilde{\theta}_0| s= 0] = \Prob[\theta > \tilde{\theta}_1| s = 1]$. The equilibria under DP are characterized by a set of quadraples $(\tilde{\theta}_0, \tilde{\theta}_1, \pi_0, \pi_1)$ that satisfies following constraints: (i) Equality \eqref{eq_applicants} holds for $s=0,1$. (ii) Moreover, letting $\tilde{\theta} := \tilde{\theta}_0 = \tilde{\theta}_1$, the following holds:
\begin{align*}
(\tilde{\theta}_0, \tilde{\theta}_1) =  \max_{(\theta_0, \theta_1)}\ & \FW(\tilde{\theta}_0, \tilde{\theta}_1, \pi_0, \pi_1), \nn
\textrm{s.t.}\ &\pi_0 (1-F_{q,s=0}) + (1- \pi_0) (1-F_{u,s=0}) \nn & =
  \pi_1 (1-F_{q,s=1}) + (1- \pi_1) (1-F_{u,s=1}).
\end{align*}
\end{definition}
In other words, it equalizes the ratio of the people accepted among $s=0,1$.
However, as discussed in Coate and Loury \cite{CL93}, such a constraint does not remove disparity:
\begin{proposition}
There exists an equilibrium with $\pi_0 \ne \pi_1$ under the demographic parity.
\label{prop_ineqaa}
\end{proposition}
The formal construction of explicit example was shown in Coate and Loury \cite{CL93} (Section B therein). Although they show some example where $\theta$ is discrete, it is not very difficult to empirically confirm that standard classifier can yields equilibria of $\pi_0 \ne \pi_1$ as we empirically show in Section \ref{sec_simulation}.
At a word, an asymmetric equilibrium exists when (i) the ratio of minority $\lambda_1$ is  small, and (ii) the classifier is very accurate (i.e., $F_{u,s}(\theta)-F_{q,s}(\theta)$ is large). 
In such a case, the firm ``patronizes'' the minority of not exerting efforts (i.e., small $\pi_1$) because it is relatively cheaper to admit a small fraction of the unqualified minorities than dismissing many qualified majority applicants. The equilibrium is discouraging minorities as they have a little motivation for investing themselves when they know they are accepted regardless of their efforts.

Recent work \cite{DBLP:conf/nips/HardtPNS16,DBLP:conf/www/ZafarVGG17} proposed alternative criteria of fairness called equalized opportunity and equalized odds. 
Let $\FP_s(\tilde{\theta}) = \Prob[\theta > \tilde{\theta}_s | s, e = 0 ]$ and $\TP_s(\tilde{\theta}) = \Prob[\theta > \tilde{\theta}_s | s, e = 1 ]$
be the false positive (FP) and the true positive (TP) rate of the classifier, respectively. 
The equalized odds criterion requires $\theta$ to have the same Receiver Operating Characteristic (ROC) curve (i.e., a curve comprised of (FP, TP)) for both groups. 
When the data is biased, $\theta$ does not satisfy the equalize odds criterion \cite{sweeney2013,misc:219,pmlr-v81-buolamwini18a}. In our simulation in Section \ref{sec_simulation}, the classifier trained with a U.S. national survey dataset is biased towards the majority (Figure \ref{fig_wwee} (a)).
To address this issue, Hardt et al.\,\cite{DBLP:conf/nips/HardtPNS16} proposed a post-processing that derives another classifier $\theta'$ from the original signal $\theta$. 
The following theorem states the feasible region of FP and TP rates of the derived predictor.
\begin{theorem}{\rm (feasible region of a derived predictor \cite{DBLP:conf/nips/HardtPNS16})}
Consider a two-dimensional convex region that is spanned by the (FP,TP)$(\theta)$-curve and a line segment from $(0,0)$ to $(1,1)$. The (FP,TP) of a derived predictor $\theta'$ lies in the convex region.
\end{theorem}
In other words, any $\theta'$ of an ROC curve is available as long as it is under the ROC curve of $\theta$. The EO policy is formalized as follows:
\begin{definition}{\rm (Equalized odds)}
The firm's policy is said to be odds-equalized when a (derived) predictor $\theta'$ satisfies
  $\FP_{s=0}(\theta') =  FP_{s=1}(\theta')$ and $\TP_{s=0}(\theta') =  TP_{s=1}(\theta')$, 
and the assignment based on the derived signal $\theta'$ is color-blind.
\label{eodds}
\end{definition}
The following theorem states that the EO does not generate disparity: There exists no asymmetric equilibrium under a derived predictor of EO.
\begin{theorem}
 For any equilibrium under EO, $\pi_0 = \pi_1$ holds.
\label{thm_eodds}
\end{theorem}
\begin{proof}
Let $\tilde{\theta}'$ be the threshold at an equilibrium.
From EO, $\tilde{c}_s(\tilde{\theta}') = \omega (F_{s,u}(\tilde{\theta}') - F_{s,q}(\tilde{\theta}'))$ is identical for two groups $s=0,1$, and thus $\pi_s = G(\tilde{c}_s(\tilde{\theta}'))$ is also identical.
\end{proof}
Note that Hardt et al.\,\cite{DBLP:conf/nips/HardtPNS16} also proposed a policy called equalized opportunity that only requires the equality of TP.
By definition, any predictor of the equalized odds satisfies the equalized opportunity, but not vice versa. Unlike the equalized odds, the equalized opportunity can result in $\pi_0 \ne \pi_1$. 

\section{Simulation}
\label{sec_simulation}
\begin{figure}[t]
    \begin{center}
        \subfloat[The ROC curve]{
            \includegraphics[width=0.3\textwidth]{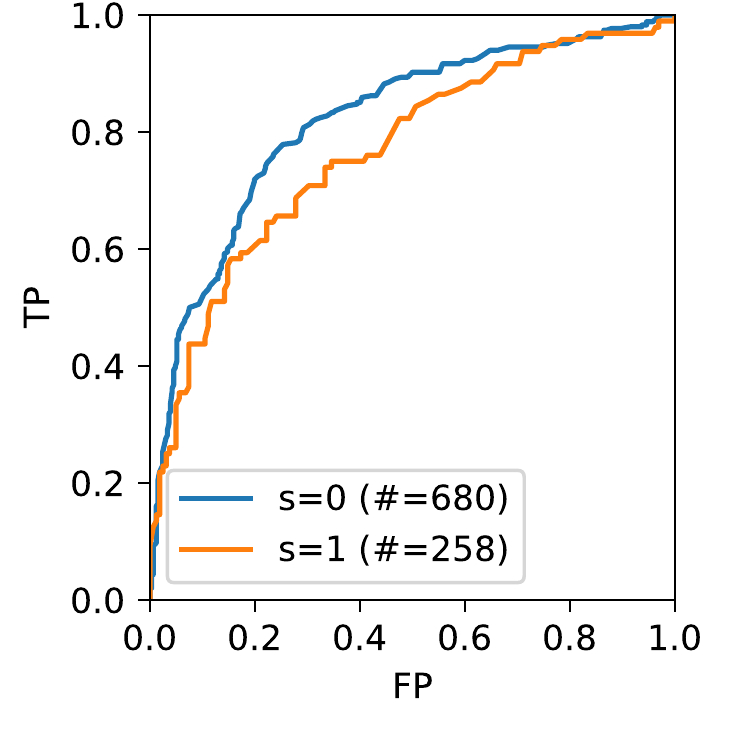}
        }
        \subfloat[$s=0$]{
            \includegraphics[scale=0.4]{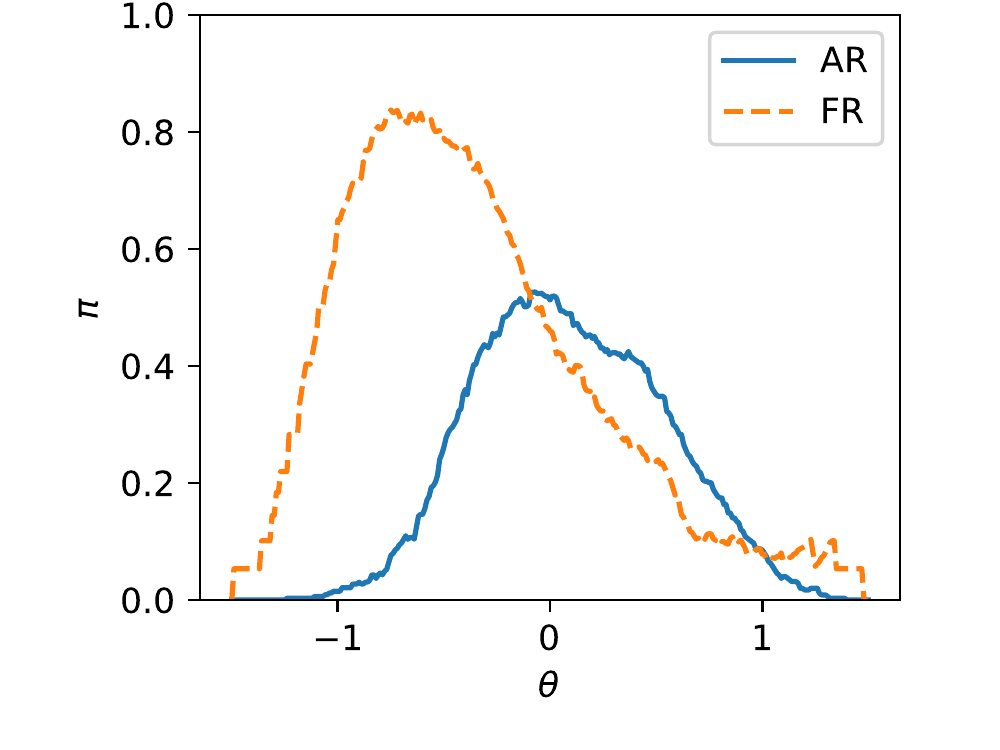}
        }
        \subfloat[$s=1$]{
            \includegraphics[scale=0.4]{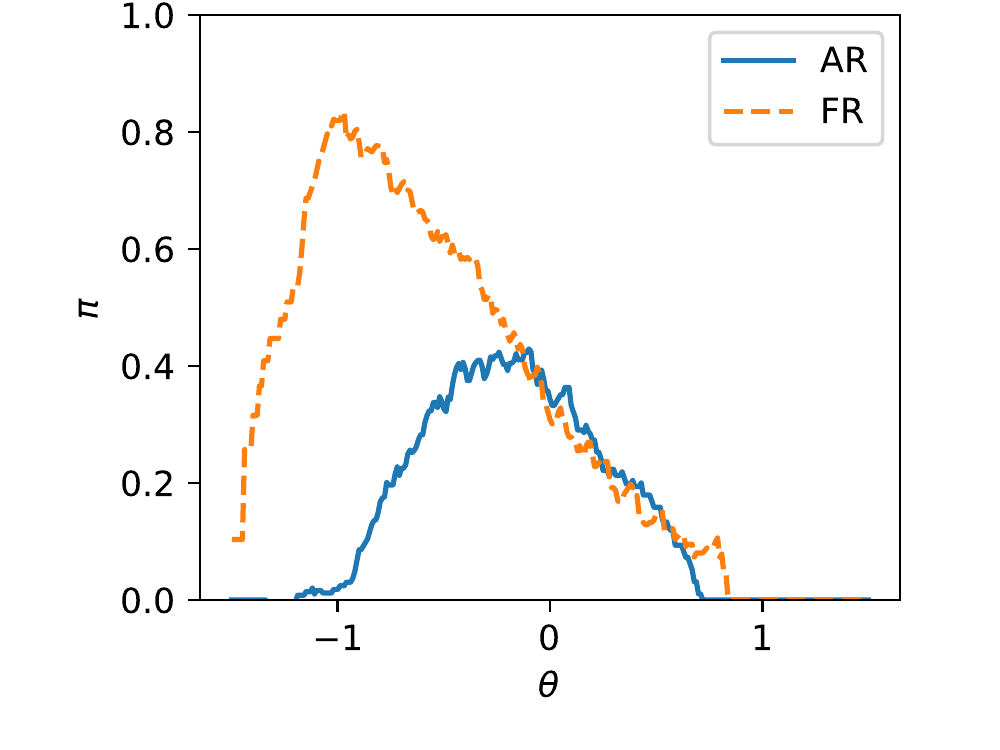}
        }
    \end{center}    
\caption{(a) The ROC curves of a predictor trained with the NLSY dataset. Details of the dataset and the settings are described in Section \ref{sec_simulation}. One can confirm that the convexity of the ROC curve is equivalent to MLRP. In the figure, the two ROC curves are fairly close to convex. (b)(c) The WR and AR curves estimated from the NLSY dataset.}
\vspace{-1.5em}
\label{fig_wwee}
\end{figure}%
To assess the social welfare and disparity on the equilibrium of the LF, CB, DP, and EO policies, we conducted numerical simulations.

\begin{table}[b!]
\caption{Results of the policies. The social welfare (SW), the welfare of the firm (FW), and disparity $|\pi_0 - \pi_1|$ of the best equilibrium are shown. We set $\lambda_0 = 2028/(2028+782)$.}
\begin{center}
  \begin{tabular}{|c|c|c|c|c|} \hline
   policy & LF & CB & DP & EO \\\hline\hline
   Disparity & $9.8$\% & $8.8$\% & $12.8$\% & $0$\%  \\\hline
   SW & $2338.6$ & $2304.7$ & $2411.2$ & $1412.2$ \\\hline
   RW & $1614.3$ & $1596.9$ & $1564.8$ & $1004.3$ \\\hline
  \end{tabular}
\end{center}
\label{tbl_results}
\end{table}%

\textbf{Dataset and Settings:} We used the National Longitudinal Survey of Youth (NLSY97) dataset retrieved from https://www.bls.gov/nls/ that involves survey results by the U.S. Bureau of Labor Statistics that is intended to gather information on the labor market activities and other life events of several groups. We model a virtual company's hiring decision assuming that the company does not have access to the applicants' academic scores.
We set $y$ to be whether each person's GPA is $\ge 3.0$ or not.
Sensitive attribute $s$ is the race of the person ($s=0$: white, $s=1$: black or African American). We have total 2,028 (resp. 782) people of $s=0$ (resp. $s=1$). and $X$ to be demographic features comprised of their school records, attitude towards life (voluntary and anti-moral activities of themselves and their peers), and geographical information during 1997 (corresponding to their late teenage).
The reward $v_q$ (resp. $v_u$) are chosen to be $53097 - 46640$ (resp. $46640 - 40604$) dollars, which is the gap of the average income in 2015 (corresponding to their early thirties) between the people of GPA $\ge 3.0$ (resp. $< 3.0$) from all people: If a job market is in perfect competition, the wage is equivalent to the productivity of the workers that a company hires, and hiring a worker yields reward that is a gap between his or her productivity and the average wage. $\omega$ are chosen to be $46640 - 40604$ dollars, which models the gap between the salary of the firm and the minimum wage they would be able to obtain with minimal effort. The cost distribution is chosen to be uniform distribution from $0$ to $\max_{s,\theta} \omega (F_{u,s}(\theta) - F_{q,s}(\theta))$, as the applicants of a cost above this value never exert effort. Note that our results are not very sensitive to these settings as long as multiple equilibria exist.
We used the RidgeCV classifier of the scikit-learn library \cite{scikit-learn} to yield $\theta$. The two thirds of the people are used to train the classifier, and the following results are tested by using the rest of them.

\textbf{Results:}
From the ROC curve is shown in Figure \ref{fig_wwee} (a) one can see that the accuracy of the classifier varies among two groups: The GPA of the majority $s=0$ is more predictable than the minority. This might come from the fact that a classifier minimizes the cumulative empirical loss, and as a result it tends to fit to the majority. 
Figure \ref{fig_wwee} (b)(c) shows the best response of the applicants and the firm under LF. Generally, equilibrium values of $\pi_0$ is larger than that of $\pi_1$. As a result, the social welfare per person in $s=0$ is usually larger than that of $s=1$. Note that, in estimating the FR curve, we applied some averaging to make $f_{e,s}(\theta)$ stable.

Based on the $F_{e,s}(\theta)$ and $f_{e,s}(\theta)$ in Figure \ref{fig_wwee}, we conducted simulation to confirm the social welfare (SW) and disparity measured by $|\pi_0 - \pi_1|$ (Table \ref{tbl_results}). In finding equilibria, we discretized $\theta$ and sought where the best response curves intersected.
One can see that (i) The result of CB and DP are more or less the same as the one of LF: They did not remove disparity. DP even increases the disparity. Unlike these policies, EO is disparity-free.
(ii) EO, which is the only policy that does not yield disparity, results in the smallest SW. This result is not very surprising because EO reduces the predictive power of the classifier for $s=0$ to match up with that for $s=1$, which we may consider as a price of achieving incentive-level equality.
Somewhat surprisingly, DP slightly increases SW, about which we discuss in Appendix \ref{app_swdec}.

\section{Conclusion}

We have studied a game between many applicants and a firm, which models human-related decision-making scenes such as university admission, hiring,  and credit risk assessments. Our framework that ties two lines of work in the theoretical labor economics and the machine learning provides a method to compare existing (or future) non-discriminatory policies in terms of their social welfare and disparity. The framework is rather simple, and many extensions can be considered (e.g., making the investment $e$ continuous value).
Although we show that EO is the only available policy that does not yield disparity, it tends to reduce social welfare. Interesting directions of possible future work include a proposal of policy that balances the social welfare and disparity: A policy with minimal or no loss of social welfare that has a small disparity is desirable. Another possible line of future work lies in evaluating policies in online settings, such as multi-armed bandits \cite{DBLP:conf/sigecom/KannanKMPRVW17}.

\clearpage
\section*{Acknowledgement}

The authors gratefully thank Hiromi Arai and Kazuto Fukuchi for useful discussions and insightful comments.

\bibliographystyle{unsrt}
\bibliography{main}
\clearpage

\appendix

\begin{center}
{\large \bf
Appendix of ``Comparing Fairness Criteria Based on Social Outcome''
}
\end{center}

\section{Does Non-discriminatory Policies Decrease Social Welfare?}
\label{app_swdec}

We first explain the reason why DP sometimes increases the social welfare.
In Figure \ref{fig_eewwcurve} of the main paper, $\tilde{\theta}_s^{(1)}$ lies in the regions where the AR curve is increasing. Intuitively, this means that around $\tilde{\theta}_s^{(1)}$ making the requirement stricter encourages the applicants to invest in skills. Compared to LF, under DP the employer imposes a milder threshold on the disadvantaged group and stricter threshold on the advantaged group. Given the advantaged group is in $\theta_s^{(1)}$, adopting DP encourages their investment, which can result in the improvement in the overall productivity of the group. When $\lambda_1$ is very small, this effect offsets the loss of the efficiency due to hiring the minorities of less productivity, which sometimes result in an improvement of SW.
As to the equalized odds, there can be some corner-case examples such that $TP_{s=0}(\theta') > TP_{s=1}(\theta')$ and $SW_0 < SW_1$ depending on the shapes of the FR and AR curves: In such a case, EO increases the social welfare.
In summary, when $\lambda_1$ is small, DP sometimes improves SW as we saw in our experiment (Table \ref{tbl_results}). EO can increase SW in some corner-case, but we think such a case is very unusual.

\end{CJK*}
\end{document}